\documentclass[runningheads]{llncs}
\usepackage{graphicx} % Required for inserting images
\usepackage{algorithm}
\usepackage{enumitem}
\usepackage{xcolor}
\usepackage{amsmath}
\usepackage{amssymb}
\usepackage{mathtools}
\usepackage[noend]{algpseudocode}
\usepackage{booktabs}
\usepackage{xspace}
\usepackage{float}
\usepackage{orcidlink}

\newcommand{\eventrule}{rule\xspace}
\newcommand{\eventrules}{rules\xspace}

\AtBeginDocument{%
  }

\author{
Jaime Osvaldo Salas\orcidlink{0000-0002-9353-8955} \and 
Paolo Pareti\orcidlink{0000-0002-2502-0011} \and 
George Konstantinidis\orcidlink{0000-0002-3962-9303}}

\institute{University of Southampton, University Road, Southampton, SO17 1BJ, UK \\ Contact: \email{j.o.salas@soton.ac.uk}, \email{p.pareti@soton.ac.uk}}

\authorrunning{Salas et al.}

\begin{document}

\title{ODRL Policy Comparison Through Normalisation}
\date{October 2025}

\maketitle

\begin{abstract}
The ODRL language has become the standard for representing policies and regulations for digital rights. However its complexity is a barrier to its usage, which has caused many related theoretical and practical works to focus on different, and not interoperable, fragments of ODRL. Moreover, semantically equivalent policies can be expressed in numerous different ways, which makes comparing them and processing them harder. Building on top of a recently defined semantics, we tackle these problems by proposing an approach that involves a parametrised normalisation of ODRL policies into its minimal components which 
%under closed-world and prohibition-by-default semantics, 
reformulates policies with permissions and prohibitions into policies with permissions exclusively, and simplifies complex logic constraints into simple ones. We provide algorithms to compute a normal form for ODRL policies and simplifying numerical and symbolic constraints. We prove that these algorithms preserve the semantics of policies, and analyse the size complexity of the result, which is exponential on the number of attributes and linear on the number of unique values for these attributes. We show how this makes complex policies representable in more basic fragments of ODRL, and how it reduces the problem of policy comparison to the simpler problem of checking if two rules are identical.

\end{abstract}

\section{Introduction}

The Open Digital Rights Language (ODRL)~\cite{iannella2018odrl} is a policy expression language, framework and vocabulary inspired by the domain of deontic logic~\cite{follesdal1971deontic} that has become the de-facto standard for expressing policies and regulations for digital rights. Using this language, users may express policies and rules declaring data access and usage rights, privacy settings, etc. 
However, as a result of the high expressivity and complexity of ODRL, there is a noticeable heterogeneity in the different interpretations of ODRL that are being considered in the literature \cite{slabbinck2025interoperable} and supported by existing tools.\footnote{\url{https://github.com/w3c/odrl/tree/master/landscape}} This makes it difficult to build upon theoretical results and achieve interoperability between applications. Moreover, semantically equivalent policies can be expressed in numerous ways, which exacerbates the policy comparison problem. 
%gheisari2025core

Previous works have attempted to define semantics for ODRL for the problems of policy comparison and evaluation. In the early days of ODRL, Pucella and Weissman~\cite{pucella2006formal} proposed a translation of formulas in a fragment of many-sorted first-order logic with equality; their work mainly focused on how to decide when a permission (or prohibition) follows from a set of ODRL statements. In another work, Steyskal and Polleres~\cite{steyskal2015towards} use rule-matching and logic programs to tackle policy comparison, especially in the presence of implicit dependencies between actions. De Vos et al.~\cite{de2019odrl} translate policies into answer set programs, to check compliance between policies and regulatory frameworks such as the GDPR. In \cite{bonatti2025towards}
authors provide a declarative semantics for ODRL which focuses on a variant of
the evaluation problem, where policies are used to determine compliance over traces (i.e. sequence of states).
More recently, Salas et al.~\cite{salas2025evaluation} have addressed the policy evaluation problem by modelling ODRL rules as queries that are evaluated on a set of \textit{events} that comprise a \textit{state of the world}.

In this work we adopt the semantics of ODRL defined by Salas et al.~\cite{salas2025evaluation} (focussing on its core fragment \emph{ODRL Lite}) as it is the most recent comprehensive semantics of the latest version of ODRL (2.2). This semantics models \emph{policies} as a triple of sets of rules --permissions, prohibitions and obligations-- over a number of attributes and the values that these attributes can have. In this paper we tackle the problem of compliance between policies by proposing a novel approach that involves a parametrised normalisation of ODRL policies into its minimal components. Under closed-world semantics, 
% and prohibition-by-default (or permission-by-default) semantics, 
this allows us to reformulate policies with permissions and prohibitions into policies with only permissions (if we assume prohibition-by-default semantics) or prohibitions (if we assume permission-by-default semantics), to simplify logic constraints (i.e. complex combinations of constraints) into simple constraints, and ensuring distinction (e.g.\ lack of overlap) between rules, avoiding cases where multiple rules apply to the same events.

This normalisation brings two key advantages. The first one is about enhancing the interoperability of ODRL tools and theoretical results. For example, our normalisation would allow a tool that outputs policies with prohibitions and logic constraints to work with a tool that only supports permissions and simple constraints as inputs. The second advantage is that this results in a set of rules that can be compared directly to other sets of rules, for example to check for containment, equivalence or overlap between policies, in a simple pair-wise fashion. This avoids the costly step of having to consider how arbitrary combinations of permission and prohibition rules interact with each other.

As the main contributions of this paper, we provide an algorithm to compute a normal form for ODRL policies and decomposing rules into their atomic components, along with a prototype implementation. We prove these algorithms preserve the semantics of policies, and analyse the size complexity of the result, which is exponential on the number of attributes and linear on the number of unique values for these attributes.

\section{Preliminaries}

An ODRL policy has four \emph{core components}: (1) \emph{Asset}: a resource that is subject of a policy; (2) \emph{Party}: an entity that undertakes functional roles as the \emph{assigner} or \emph{assignee} of a policy; (3) \emph{Action}: an operation (for example Read, Write, Copy, Anonymise) that can be exercised on the target asset and (4) \emph{Rule}: defines the \emph{permission}, \emph{prohibition}, or \emph{obligation} of the assignees of the policy to execute actions over assets. A single policy may define multiple rules on different core components. 
Complex ODRL policies can also make use of additional types of rules, such as \emph{duties}, \emph{remedies} and \emph{consequences}, that extend the meaning of the three basic types previously mentioned, but we omit those from our analysis as these fall outside of the scope of our normalisation approach. 
ODRL policies are usually read as \emph{\underline{Rule} for \underline{Assignee} to \underline{Action} \underline{Asset} subject to \underline{Constraints} and \underline{Refinements}}, \emph{e.g.} \emph{Permission for Alice to Read Document D subject to the constraint the day being Monday and Document D to be refined to its Subset of Pages 1-10}. Except Action, all other elements of a rule are optional. 

We now briefly introduce core concepts of ODRL semantics, but we refer the reader to Salas et al.~\cite{salas2025evaluation} for more details. As the main building blocks of a policy, ODRL \emph{rules} are sets of \emph{constraints}, which in turn are boolean combinations of \textit{conditions} of the form $(\lambda \ op\ \rho)$. The left operand $\lambda$ is a property that is drawn from a known vocabulary, or Knowledge Graph, whose value is compared to the right operand $\rho$ by the operator $op$. The operator $op$ can be a binary operator between constants ($=$, $>$, $\geq$, $<$, $\leq$, $\neq$), a binary operator between sets ($\equiv$, $\in$, $\notin$, $\supset$, $\subset$), or a special operator that checks class membership ($=_{type}$). The rule from our example could be formalised as the following constraints: $p_{Read}: \; \{ (Party = Alice) , (Action = Read) , (Asset = Document\ D) , (Day = Monday) , (Pages \geq 1) , (Pages \leq 10) \}$. 

These constraints are evaluated on sets of \textit{events}, where each event $e$ is a tuple of fixed length, whose indexes represent specific ODRL core components and/or left operands. With the exception of the index representing the action, events can have null values. For example, assume a mapping between the first five indexes of an event and the features Party, Action, Asset, Day and Page, respectively. Under this mapping, events $e_{Alice}: \;<Alice, Read, Document\ D, \\Monday, 5>$ and $e_{Bob}:\; <Bob, Edit, Document\ D, Tuesday, null>$ represent, intuitively, the fact that Alice read page 5 of Document D on Monday, and Bob edited it on Tuesday (without specifying which pages).

We say that a rule $\tau$ \emph{matches} an event $e$, denoted by $\mathsf{match}(\tau, e)$ if all its constraints evaluate to true on $e$, that is when $(\bigwedge_{c \in \tau}c)(e)$ is true. 
Thus, when evaluating a rule, all of its constraints are conjuncted with each other.  
We will formalise later the concept of validity of a state of the world with respect to a policy. Intuitively, when a policy only contains permissions and prohibitions, a compliant state of the world must contain only permitted events, and no prohibited event. 
In our running example, state of the world $\{e_{Alice},e_{Bob}\}$ would violate a policy that has permission $p_{Read}$ as the only rule, because the event $e_{Bob}$ is not explicitly permitted, and thus it is implicitly prohibited. 
It is also important to note that the chosen semantics does not include semantic reasoning, that is the inference of new facts, and constants are assumed to be disjoint from each other, following the unique name assumption.

\section{Normalisation}

In this section, we describe our normalisation process, which involves two main steps. First, the regularisation of ODRL constraints, which involves computing a normal form for constraints based in Boolean algebra. Second, the splitting of intervals, which involves breaking down intervals defined by constraints into minimal intervals according to a set of values (right operands).

\subsection{Constraint Regularisation}

In ODRL, the constraints that apply to a rule are defined in two ways. 
First, by specifying the IRI of the Action, Asset or Assignee of the rule. In the chosen semantics, these are equivalent to simple equality conditions, e.g. $(Action = Read)$. 
The second way is by specifying boolean combinations of conditions, that refine the semantics of an Action, Party or Asset, or declare the conditions applicable to a Rule. This effectively defines a subset of the element. For example, a party collection could be defined as ``customers over 60'', which is a subset of all customers. 
For this work, we identify important sub-classes of constraints in the following list, with indentation denoting class membership.

\begin{description}

  \item[\textbf{Simple constraints}] 
    Constraints of the form $(\lambda \; op \; \rho)$ (e.g., $(Asset = Book)$).
    \begin{description}
      \item[\textbf{Numeric constraints}] 
        when $\lambda$ is numeric (e.g., $(Age > 18)$).
        \begin{description}
          \item[\textbf{Equality constraints}] 
            when $op$ is $=$ (e.g., $(Count$ $= 5)$).
          \item[\textbf{Inequality constraints}] 
            when $op$ is $<$ or $>$ (e.g., $(Age > 21)$).
        \end{description}
        \item[\textbf{Set constraints}] 
        when $op$ compares sets (e.g., $(Formats$ $\supseteq \{CSV,XML\})$)  or values to sets (e.g., $(Location \in \{Library,Park\})$).
      \item[\textbf{Entity constraints}] Symbolic constraints on entities that are neither numeric nor sets. They only allow the operators $=$, $\neq$ and $=_{type}$ (e.g., $(Jane =_{type} Researcher)$).
    \end{description}

  \item[\textbf{Logical constraints}] 
    Boolean combinations of simple constraints using the $\wedge$, $\vee$ and $\neg$ operators
    (e.g., $(Age>18) \wedge (Date > \text{``12--12--2012''})$).\footnote{ODRL uses exclusive-or $\oplus$ instead of $\neg$, but for simplicity we assume equivalent expressions: $\neg c \iff c \oplus (c=0 \vee c \neq 0 )$ and $c \oplus d \iff (c \wedge \neg d) \vee (d \wedge \neg c )$.}
    \begin{description}
      \item[\textbf{Interval constraints}] 
        Boolean combinations of numeric constraints 
        on the same $\lambda$ (e.g., $(Age>18) \wedge (Age<65)$).
    \end{description}

\end{description}

\noindent Because of the high expressivity of ODRL, the same rule can be syntactically expressed in several ways while matching the same events. For example, a permission rule that states that a client can play a movie if they are over 18 and paid 10 EUR, or 5 EUR if they are under 18, can be expressed as the following rule using a logic constraint. 
\begin{equation}
\begin{aligned}
 R&=\{ (Action=Play) , (Asset=Movie) , (( (Age \geq 18) \wedge (Payment=10\text{EUR})) \\ &\phantom{--------}\vee ( (Age < 18) \wedge (Payment=5\text{EUR}))  ) \}   
\end{aligned}
\end{equation}
Alternatively, the same permission can be expressed with the following two permission rules, each containing only simple constraints.
\begin{equation}
\begin{aligned}
 R_1 &=\{ (Action=Play) , (Asset=Movie) , (Age \geq 18) , (Payment=10\text{EUR}) \}  \\
  R_2 &=\{ (Action=Play) , (Asset=Movie) , (Age < 18) , (Payment=5\text{EUR}) \}
\end{aligned}
\end{equation} 
More complex cases can arise with other binary relations, boolean equivalences and additional rules. In this section, we describe the steps to compute a normal form for rules that preserves their semantic meaning to simplify checking for equivalence and containment.
Our approach focuses on constraints with equalities and inequalities between numeric values, and symbolic equivalences.

Set constraints require additional assumptions, and differ in whether we compare sets against sets or values to sets.  
Entity constraints with the $=_{type}$ operator additionally require the existence of an ontology and a process to determine class membership, such as %the one defined by 
OWL~\cite{antoniou2009web} classes or sub-classes.
Due to their more ambiguous interpretation we exclude set comparisons from our work, and from now on we will assume rules do not contain any set constraints, nor the $=_{type}$ operator.

Given a rule $R$, we denote with $\mathrm{N}(R)$ the normal form of $R$ computed with the following three steps. 
First, all the constraints $\{c_1, c_2, ..., c_n\}$ of $R$ are converted into a single logical constraint $C = \{c_1 \wedge c_2 \wedge ... \wedge c_n\}$. 
Second, any numeric or entity constraint, or their negation, in $C$ is normalised into equivalent constraints that contain only the operators $=$, $<$ or $>$, as shown in Table \ref{tab:constraints_reformulated}. 
Lastly, the rule is rewritten in standard disjunctive normal form expansion, turning $C$ into a disjunction of conjunctions of conditions. 

\vspace{-0.5cm}

\begin{table}[H]
\centering
\begin{tabular}{c @{\hspace{1em}} c @{\hspace{1em}} | @{\hspace{1em}} c @{\hspace{1em}} c}
Constraint & Reformulation & Constraint & Reformulation \\
\hline
$\neg(\lambda=\rho)$ & $(\lambda < \rho) \vee (\lambda > \rho)$
& $\neg(\lambda=\rho)^*$ & $(\lambda \neq \rho)$\\
$\neg(\lambda \ge \rho)$ & $(\lambda < \rho)$ 
& $(\lambda \neq \rho)$ & $(\lambda < \rho) \vee (\lambda > \rho)$\\
$\neg(\lambda < \rho)$ & $(\lambda > \rho) \vee (\lambda = \rho)$
& $\neg(\lambda \le \rho)$ & $(\lambda > \rho)$ \\
$(\lambda \le \rho)$ & $(\lambda < \rho) \vee (\lambda = \rho)$
& $\neg(\lambda > \rho)$ & $(\lambda < \rho) \vee (\lambda = \rho)$ \\
$\neg(\lambda \neq \rho)^\dagger$ & $(\lambda = \rho)$
& $(\lambda \ge \rho)$ & $(\lambda > \rho) \vee (\lambda = \rho)$\\
\end{tabular}
\caption{Constraint reformulations for numeric constraints, (*) applies also to entity constraints. ($\dagger$) applies only to entity constraints. }
\label{tab:constraints_reformulated}
\end{table}
\vspace{-1cm}
\begin{theorem}
    Given a \eventrule $R$ and an event $e$, $\mathsf{match}(R,e)$ is true if and only if $\mathsf{match}(\mathrm{N}(R),e)$ is true.
    \label{theorem:DNF}
\end{theorem}
\begin{proof}
The last two steps of the regularisation, namely the expansion to disjunctive normal form, and the simple reformulations in Table \ref{tab:constraints_reformulated} preserve logic equivalence, and thus do not affect the evaluation of the $\mathsf{match}$ operation.
For the first step, note that if a rule has multiple constraints, they are evaluated conjunctively by the $\mathsf{match}$ operation, so equivalently as if they were joined into a single conjunctive constraint. 
\end{proof}

We denote rules that contain only simple constraints as \textit{simple \eventrules}. Given a rule $\tau$ that contains only a conjunction of simple constraints $c_1 \wedge c_2 \wedge ... \wedge c_k$, we denote $\bar{\tau}$ the simple rule that contains the set of simple constraints $\{c_1, c_2, ... c_k\}$. As mentioned in the proof of Theorem \ref{theorem:DNF}, since all the constraints of a rule are evaluated conjunctively, $\tau$ is trivially equivalent to $\bar{\tau}$. 
Note that there is a very clear correspondence between $\mathrm{N}(R)$ and the set of simple \eventrules that comprise it.

\begin{definition}
    Let $R$ be a \eventrule. $\mathrm{D}(R)$ is the set $\{\bar{\tau_1}, \bar{\tau_2}, \dots, \bar{\tau_k}\}$ of simple \eventrules in the disjunctive normal form of $R$, $\mathrm{N}(R) = \tau_1 \vee \tau_2 \vee \dots \vee \tau_k$.
    \label{def:D}
\end{definition}

\begin{corollary}
    Given a rule $R$ and an event $e$, $\mathsf{match}(R,e)$ is true if and only if  $\mathsf{match}(\tau,e)$ is true for some $\tau \in D(R)$.
    \label{corollary:D_valuation_preservation}
\end{corollary}
\begin{proof}
    By Theorem \ref{theorem:DNF} $\mathsf{match}(R,e)$ is true if and only if one of the disjuncts $\tau$ of $N(R)$, and thus its corresponding rule $\bar{\tau}$ of the decomposition $D(R)$, matches $e$.
\end{proof}

\subsection{Interval Constraint Splitting}

To illustrate the need to split intervals, let $R$ be a \eventrule that includes the constraint $c=((Age > 18) \wedge (Age \leq 33)) \vee ((Age > 33) \wedge (Age < 65))$ and $R'$ be a rule that includes the constraint $c'=(Age > 21) \wedge (Age < 45)$. The former says that the ``age'' property should be between 18 and 33, inclusive, or over 33 and under 65, while the latter simply says that the property should be between 21 and 45. Intuitively, the latter should be fully contained in the former because the former is essentially the same as a continuous interval from 18 to 65, and 21 to 45 is clearly within those bounds. However, $\mathrm{D}(R)$ would contain three simple \eventrules $\tau_1$, $\tau_2$ and $\tau_3$ that include the constraints $c_1=(Age > 18) \wedge (Age < 33)$, $c_2=(Age > 18) \wedge (Age = 33)$ and $c_3=(Age > 33) \wedge (Age < 65)$. If we were to compare $R'$ against any of the simple \eventrules, none of them would overlap entirely; we would need to compare $R'$ against all combinations of $\tau_1$, $\tau_2$ and $\tau_3$, which is exponential on the number of simple \eventrules. Instead, our approach ``splits'' the interval defined by $R$ into components that can be directly compared against $R'$.

\subsubsection{Simplifying Intervals}
In the previous example, $c_2$ defines an interval $(Age > 18) \wedge (Age = 33)$, where the simple constraint $(Age > 18)$ is redundant because the equality constraint is stricter and more specific. This illustrates that intervals can contain redundant elements. 

It follows that any conjunction of constraints $\bigwedge (\lambda \ op_i \ \rho_i)$ for a single numeric operand $\lambda$ where each $op_i$ is either $=$, $<$ or $>$ can be rewritten to an equivalent constraint with at most 2 constraints if the maximum lower bound $\rho_{min}$ is less than the minimum upper bound $\rho_{max}$ and either there are no constraints that are equality constraints, or if one of the constraints $c$ is an equality $(\lambda=\rho)$ with $\rho \in [\rho_{min},\rho_{max}]$, and there are no other equalities $(\lambda=\rho')$ s.t. $\rho \neq \rho'$.

In such a case, we can simplify $c$ by replacing it with a logical constraint of the form $(\lambda > \rho_{min}) \wedge (\lambda < \rho_{max})$, where $\rho_{min}$ is the maximum lower bound and $\rho_{max}$ is the minimum upper bound; in cases where no upper bound is specified, then $c=(\lambda > \rho_{min})$, and if no lower bound is specified, then $c=(\lambda < \rho_{max})$. For ease of notation, we will use $(\lambda \in (a,b))$ as a shorthand notation for $(\lambda > a) \wedge (\lambda < b)$ and replace either $a$ or $b$ with $\pm \infty$ if either bound is not specified.  In any other case, the interval is empty, and no event can match the constraint. In such a case, we rewrite any rule with an empty interval to a \textit{canonical false rule} $R_\perp$, for instance $R=(Action=Use)\wedge(x>0)\wedge(x < 0)$.

\subsubsection{Splitting intervals}

Let $R_1$ be a rule that states that left operand $\lambda$ is true for values within the intervals $(a,b)$ and $(c,d)$, and another rule $R_2$ that constrains $\lambda$ with the right operands $e$, $f$ and $g$ such that $(a < e <b)$, $(b < f < c)$ and $(g > d)$. To properly compare the two rules, their normalisation must take into consideration the constant values (i.e. right operands) found in both $R_1$ and $R_2$.
We can achieve this by ``splitting'' intervals by the right operands in both rules, such that the intervals $(\lambda \in (a,b))\vee (\lambda \in (c,d))$ are reformulated as the intervals $(\lambda \in (a,e))\vee(\lambda = e)\vee(\lambda \in (e,b))\vee(\lambda \in (c,d))$; since both $f$ and $g$ are values found outside either intervals, these do not appear in new intervals.

Let $V(R,\lambda)$ be a function that, for a given rule and operand, returns the set of right operands that appear in constraints that have $\lambda$ as a left operand, and $V(R)$, the function that returns all the right operands in the constraints of $R$. More formally, $V(R,\lambda)=\{\rho \mid (\lambda\ op \ \rho) \in R\}$. Similarly, let $\Lambda(R) = \{\lambda \mid (\lambda \ op \ \rho) \in R \}$ and $C(R,\lambda)=\{c \mid c \in R \text{ such that }c  = (\lambda \ op \ \rho) \}$. Finally, let $\mathrm{S}(R,V)$ denote the result of splitting the intervals of $R$ according to $V$.

\begin{algorithm}[h!]
    \begin{algorithmic}[1]
        \Require $R_1$, $R_2$ conjunctions of conditions.
        \Function{SplitIntervals}{$R_1,R_2$}
        \For{$\lambda \in \Lambda(R_1) \cup \Lambda(R_2)$} \label{alg:iterateLambda}
        \State $V \gets sort(V(R_1,\lambda) \cup V(R_2,\lambda))$ \label{alg:values}
        \If{$C(R_1,\lambda)=\emptyset$} \label{alg:emptyCheck}
        \For{$v \in V$}
        % \State $R_1 \gets R_1 \wedge ((\lambda < v) \vee (\lambda = v) \vee (\lambda > v) \vee (\lambda \neq v))$ 
        \If{$\lambda$ is numeric}
        \State $R_1 \gets R_1 \wedge ((\lambda < v) \vee (\lambda = v) \vee (\lambda > v))$
        \Else
        \State $R_1 \gets R_1 \wedge ((\lambda=v) \vee (\lambda \neq v))$
        \EndIf
        \EndFor
        \Else
        \If{$C(R_1,\lambda)$ is a non-empty interval} \label{alg:intervalCheck}
        \State $\rho_{min} \gets -\infty$
        \State $\rho_{max} \gets \infty$
        \If{$\exists\rho \text{ such that }(\lambda > \rho) \in C(R_1,\lambda)$} \label{alg:lowerBound}
        \State $\rho_{min} \gets \rho$
        \EndIf
        \If{$\exists\rho \text{ such that }(\lambda < \rho) \in C(R_1,\lambda)$} \label{alg:upperBound}
        \State $\rho_{max} \gets \rho$
        \EndIf
        \State $I \gets \{ v \in V \mid (v > \rho_{min}) \wedge (v < \rho_{max}) \}$ \label{alg:interval}
        \If$I$ is not empty
        \State $c \gets (\lambda \in (\rho_{min},I_1))$\label{alg:firstInterval}
        \For{$i \in [1,|I|-1]$}
        \State $c \gets c \vee (\lambda \in (I_i,I_{i+1})) \vee (\lambda = I_i)$
        \EndFor
        \State $c \gets c \vee (\lambda \in (I_{|I|},\rho_{max})) \vee (\lambda = I_{|I|})$\label{alg:lastInterval}
        \State $R_1 \gets R_1 \setminus C(R_1,\lambda)$ \label{alg:removeConstraints}
        \State $R_1 \gets R_1 \wedge c$ \label{alg:addNewConstraints}
        \EndIf
        \Else
        \State return $R_\perp$
        \EndIf
        \EndIf
        \EndFor
        \State return $\mathrm{N}(R_1)$ \label{alg:lastStep}
        \EndFunction
    \end{algorithmic}
    \caption{}
    \label{alg:splitIntervals}
\end{algorithm}

Algorithm \ref{alg:splitIntervals} shows a pseudo-code implementation of the algorithm to split intervals for a rule $R_1$ according to the values found in another rule $R_2$. This algorithm iterates over all the left operands that appear in both $R_1$ and $R_2$ (Line \ref{alg:iterateLambda}), and for each left operand $\lambda$, we compute an ordered list of right operands that appear in constraints with $\lambda$ in both $R_1$ and $R_2$ (Line \ref{alg:values}). This algorithm can be naturally extended to consider multiple rules, and thus policies, at once. 
In Line \ref{alg:emptyCheck} we check if \eventrule $R_1$ has any constraints that mention left operand $\lambda$ because if not, this means that any value is valid for that left operand. We reformulate the constraint ``accepts any value'' into a disjunction of simple constraints $R_1 \gets R_1 \wedge ((\lambda < v) \vee (\lambda = v) \vee (\lambda > v))$ if numeric, or $R_1 \gets R_1 \wedge ((\lambda=v) \vee (\lambda \neq v))$ if symbolic, for every right operand $v$ over $\lambda$.

On the other hand, in Line \ref{alg:intervalCheck} we check if the constraints that mention $\lambda$ define a non-empty interval. Following the regularisation and simplification steps earlier, $R_1$ should contain at most 2 inequality constraints: one defining a lower bound, and one defining an upper bound. We check in Lines \ref{alg:lowerBound} and \ref{alg:upperBound} if they exist, in which case we assign their values to the lower bound $\rho_{min}$ or upper bound $\rho_{max}$, respectively. Then, in Line \ref{alg:interval} we compute the ordered list $I$ of values that are between the lower and upper bounds; if $I$ is empty, then we continue with the next left operand.

In Lines \ref{alg:firstInterval}-\ref{alg:lastInterval} we compute the lower interval first, from the lower bound to the first value in $I$. Then, we iterate over the values in $I$, adding an additional interval for each pair of values. Finally, we compute the upper interval, from the last value in $I$ to the upper bound. This results in a disjunction of inequality constraints $c$, which replaces the original constraints $C(R_1,\lambda)$ (Lines \ref{alg:removeConstraints}-\ref{alg:addNewConstraints}).

Note that in Line \ref{alg:addNewConstraints}, the result is a conjunction of $R_1$ and a disjunction of constraints $c$, which is no longer in its disjunctive normal form. Therefore, we require an additional regularisation step (Line \ref{alg:lastStep}). 

We developed a prototype implementation of our algorithm \cite{salas_trejo_2026_18862527}. We provide a simple command line interface that allows users to normalise ODRL policies, split intervals according to a set of constant values or fully compare two policies using normalisation. In addition, we provide a few examples of ODRL policies that can be used as input. Experimental results and performance evaluations are outside the scope of this paper and left for future work.

\begin{theorem}
    Given a set of right operands V, a simple \eventrule $R$ and an event $e$, $\mathsf{match}(R,e)$ is true if and only if $\mathsf{match}(\mathrm{S}(R,V),e)$ is true.
    \label{theorem:splitting}
\end{theorem}

\begin{proof}
    If $V$ contains a left operand $\lambda$ that does not appear in $R_1$, $R_1$ will be extended to $R_1'=R_1\wedge((\lambda < v) \vee (\lambda = v) \vee (\lambda > v) \vee (\lambda \neq v))$.
    Evidently, the extended intervals are true for any possible value for $\lambda$, so $R_1'$ is equivalent to $R_1$. Otherwise, note that intervals are only split into an equivalent disjunction of sub-intervals, and thus do not change the rule with respect to $\mathsf{match}(R,e)$. 
    
\end{proof}
Finally, we can summarise the results and apply to the full normalisation.

\begin{theorem}
    Given a set of right operands $V$, a \eventrule $R$ and an event $e$, it is true that $\mathsf{match}(R,e)$ if and only if $\mathsf{match}(\mathrm{S}(\mathrm{N}(R),V),e)$ is true. 
    \label{theorem:summary}
\end{theorem}

\begin{proof}
    This follows directly from Theorems \ref{theorem:DNF} and \ref{theorem:splitting}: $\mathsf{match}(\mathrm{S}(\mathrm{N}(R),V),e)$ is the composition of two operations that preserve the $\mathsf{match}$ evaluation of a rule. 
\end{proof}

\section{Comparing Rules and Policies}

In this section, we describe how the normalisation procedure has useful properties for deciding if a rule matches an event, which can additionally be used to check for equivalence, containment or overlap between rules. We begin by noting how rule matching can be decided by checking a set of simple rules for a match.

\begin{corollary}
    Given a set of right operands $V$, a \eventrule $R$ and an event $e$, it is true that $\mathsf{match}(R,e)$ if and only if there exists a simple \eventrule $\tau \in \mathrm{D}(\mathrm{S}(\mathrm{N}(R),V))$ such that $\mathsf{match}(\tau,e)$ is true.
    \label{corollary:summary}
\end{corollary}

\begin{proof}
This follows directly from Theorem \ref{theorem:summary} and Corollary \ref{corollary:D_valuation_preservation}.
\end{proof}

Intuitively, two rules are equivalent if they match the same events. In this section, we characterise rule overlap, containment and equivalence as a result of finding common simple \eventrules between normalised and split rules.

\begin{equation}
    \begin{aligned}
    R &= (\lambda_1=A) \wedge (\lambda_2=B) \wedge (\lambda_3 = C) \wedge (\lambda_4 \in (a,b))\\
    &\phantom{----------}\downarrow \mathrm{S}(\mathrm{N}(R),V)\\
    &\phantom{\coloneq} (\lambda_1=A) \wedge (\lambda_2=B) \wedge (\lambda_3 = C) \wedge (\lambda_4 \in (a,e))\\
    &\phantom{\coloneq} (\lambda_1=A) \wedge (\lambda_2=B) \wedge (\lambda_3 = C) \wedge (\lambda_4 \in (e,b))\\
    \\
    R' &= (\lambda_1=A) \wedge (\lambda_2=B) \wedge (\lambda_3 = C) \wedge (\lambda_4 \in (e,g))\\
    &\phantom{----------}\downarrow \mathrm{S}(\mathrm{N}(R'),V)\\
    &\phantom{\coloneq} (\lambda_1=A) \wedge (\lambda_2=B) \wedge (\lambda_3 = C) \wedge (\lambda_4 \in (e,b))\\
    &\phantom{\coloneq} (\lambda_1=A) \wedge (\lambda_2=B) \wedge (\lambda_3 = C) \wedge (\lambda_4 \in (b,g))\\
    \end{aligned}
    \label{eq:p_vs_p}
\end{equation}

As an example, Equation \ref{eq:p_vs_p} shows two rules $R$ and $R'$ after splitting intervals. We can see that $R$ overlaps with $R'$ because both rules have a constraint $c=(\lambda_4 \in (e,b))$ that matches the interval $(e,b)$ exactly. On the other hand, neither $R$ nor $R'$ is contained in the other because $R$ allows the interval $(a,e)$, which doesn't appear in $R'$, whereas $R'$ allows $(b,g)$, which doesn't appear in $R$.

Notably, the intervals in Equation \ref{eq:p_vs_p} do not overlap at all. Intuitively, this means that any event that matches the rule can only match one of its split rules.

\begin{theorem}
    Let $R$ be a rule and $V$ be a superset of the right operands $V(R)$. 
    %be a set of values that contains all $\rho \in V(R,\lambda)$ for all $\lambda \in \Lambda(R)$. 
    For all $\tau_i$ and $\tau_j$ in $\mathrm{D}(\mathrm{S}(\mathrm{N}(R),V))$ and events $e$, if $\mathsf{match}(\tau_i,e)$ is true then for all other $\tau_j \neq \tau_i$, $\mathsf{match}(\tau_j,e)$ is false.
    \label{theorem:overlap}
\end{theorem}

\begin{proof}
Assume that an event $e$ exists such that $\mathsf{match}(\tau_i,e)$ and $\mathsf{match}(\tau_j,e)$ are true, and $\tau_i \neq \tau_j$. By construction, $\tau_i$ and $\tau_j$ are conjunctions of simple constraints $(\lambda = \rho)$, $(\lambda < \rho)$, $(\lambda > \rho)$ and $(\lambda \neq \rho)$. Since $\mathsf{match}(\tau_i,e)$ is true, then for all equivalence constraints in $\tau_i$, $\lambda_m=e_m$. Therefore, without loss of generality, let $\lambda_k$ be a left operand such that, since $\tau_i$ and $\tau_j$ are not equal, there exists a numeric constraint $(\lambda_k < \rho_k)$ in $\tau_j$ for left operand $\lambda_k$ that is an inequality. We can see three cases here:
\begin{itemize}
    \item $\tau_i$ has an equivalence constraint $(\lambda_k = e_k)$. In such a case, if $(\lambda_k = e_k)$ and $(\lambda_k < \rho_k)$ are true then $e_k < \rho_k$. However, if that were true, and since $V$ contains all values in $R$, then the interval $(\lambda_k < \rho_k)$ would have been split into $(\lambda_k < e_k) \vee (\lambda_k = e_k) \vee (\lambda_k \in (e_k,\rho_k))$. Therefore, either $e_k \notin V$, which is a contradiction, or $\mathrm{S}(\mathrm{N}(R),V)$ is not normalised.
    \item $\tau_i$ has an inequality constraint $(\lambda_k \in (r^-_k,r^+_k))$ that overlaps with $(\lambda_k < \rho_k)$. This is true, if, and only if, $r^-_k < \rho_k$. However, if that were true, then the interval $(\lambda_k < \rho_k)$ would have been split into $(\lambda_k < r^-_k) \vee (\lambda_k \in (r^-_k,\rho_k)) \vee (\lambda_k \in (\rho_k,r^+_k))$ if $\rho_k < r^+_k$ or $(\lambda_k < r^-_k) \vee (\lambda_k \in (r^-_k,r^+_k)) \vee (\lambda_k \in (r^+_k,\rho_k))$ if $\rho_k > r^+_k$. Therefore, either $r^-_k \notin V$, which is a contradiction, or $\mathrm{S}(\mathrm{N}(R),V)$ is not normalised.
    \item $\tau_i$ has an equivalence constraint $(\lambda_k = e_k)$ and $e_k$ is a symbolic value. In such a case, $\tau_j$ can only match $e$ if it is also an equivalence constraint $(\lambda_k = e_k)$, which is a contradiction because $\tau_i \neq \tau_j$. The same can be argued if $\tau_i$ is an inequality constraint $(\lambda_k \neq e_k)$.
\end{itemize}
\end{proof}
From Theorem \ref{theorem:overlap} it follows that all simple \eventrules in $\mathrm{D}(\mathrm{S}(\mathrm{N}(R),V))$ are disjoint. Together with Theorem \ref{theorem:summary}, we can also conclude that $R$ matches an event $e$ if and only if there exists a \eventrule $\tau$ in $\mathrm{D}(\mathrm{S}(\mathrm{N}(R),V))$ such that $\mathsf{match}(\tau,e)$ is true and for all other $\tau'$ in $\mathrm{D}(\mathrm{S}(\mathrm{N}(R),V))$, $\mathsf{match}(\tau',e)$ is false.

Moreover, now that we can decide $\mathsf{match}(R,e)$ by deciding $\mathsf{match}(\tau,e)$ for some simple rule $\tau$ in the decomposition of $R$, we can extend this result to characterise the matching of an event to multiple rules by their sharing of simple rule in their decompositions.

\begin{theorem}
    Let $R$ and $R'$ be rules, $e$ be an event and $V$ be a set of right operands such that $V\supset\mathrm{V}(R) \cup \mathrm{V}(R')$. Both $\mathsf{match}(R,e)$ and $\mathsf{match}(R',e)$ are true if and only if there exists a \eventrule $\tau$ in both $\mathrm{D}(\mathrm{S}(\mathrm{N}(R),V))$ and $\mathrm{D}(\mathrm{S}(\mathrm{N}(R'),V))$ such that $\mathsf{match}(\tau,e)$ is true.
    \label{theorem:sameRuleComparison}
\end{theorem}

\begin{proof} $\Leftarrow$ If there exists a \eventrule $\tau$ in both $\mathrm{D}(\mathrm{S}(\mathrm{N}(R),V))$ and $\mathrm{D}(\mathrm{S}(\mathrm{N}(R'),V))$ such that $\mathsf{match}(\tau,e)$ is true, then by Corollary \ref{corollary:summary} both $\mathsf{match}(R,e)$ and $\mathsf{match}(R',e)$ are true.

$\Rightarrow$ Let $\{\tau_1, \tau_2, \dots, \tau_k\}$ be the set of rules in $\mathrm{D}(\mathrm{S}(\mathrm{N}(R),V))$ and $\{\tau'_1, \tau'_2, \dots, \tau'_k\}$ be the set of rules in $\mathrm{D}(\mathrm{S}(\mathrm{N}(R'),V))$.
If both $\mathsf{match}(R,e)$ and $\mathsf{match}(R',e)$ are true, then constraint $\tau^{\vee} = \tau_1 \vee \tau_2 \vee \dots \tau_k \vee \tau'_1 \vee \tau'_2 \vee \dots \tau'_k$ must be true. Let $R^{V}$ be the rule with constraint $\tau^{\vee}$. Notice that $\tau^{\vee}$ is already in normal form with respect to $V$, that is, $\mathrm{D}(\mathrm{S}(\mathrm{N}(R^{V}),V))$ is $\{\tau_1 , \tau_2 , \dots \tau_k , \tau'_1 , \tau'_2 , \dots \tau'_k\}$, because each disjunct in $\tau^{\vee}$ is already a conjunction of simple constraints, over intervals that have already been split according to $V$.
If a \eventrule $\tau$ in both $\mathrm{D}(\mathrm{S}(\mathrm{N}(R),V))$ and $\mathrm{D}(\mathrm{S}(\mathrm{N}(R'),V))$ such that $\mathsf{match}(\tau,e)$ is true does not exist, then there must exist two event rules $\tau_i$ and $\tau_j$ in $\tau^{\vee}$ that match $e$, such that $\tau_i \neq \tau_j$, but this would contradict Theorem \ref{theorem:overlap}.
\end{proof}

\noindent Theorem \ref{theorem:sameRuleComparison} allows us to conclude that an event is matched by two rules if and only if their normalised forms share a simple \eventrule that matches the event. Since two rules $R$ and $R'$ are said to overlap (denoted $R \sqcap R'$) if there exists an event that is matched by both, this result can be used directly to decide rule overlap. Similarly, we can characterise the containment of policies by the presence of common \eventrules, which allows us to decide rule containment ($R \sqsubseteq R'$) and equivalence ($R \equiv R'$) by checking if they share simple \eventrules. To do so, we need to remove inconsistent rules, that is, rules that do not match any event, from the result of the normalisation. Let $\dot{\mathrm{D}}(R)$ be the set of rules in $\mathrm{D}(R)$ that do not contain any empty intervals.

\begin{corollary}
    Let $R$ and $R'$ be rules, let $\mathrm{V}$ be a set of right operands such that $V\supset\mathrm{V}(R) \cup \mathrm{V}(R')$. 
    $R$ is contained in $R'$ ($R \sqsubseteq R'$) if and only if all \eventrules in $\dot{\mathrm{D}}(\mathrm{S}(\mathrm{N}(R),V))$ also appear in $\dot{\mathrm{D}}(\mathrm{S}(\mathrm{N}(R'),V))$ (denoted $\dot{\mathrm{D}}(\mathrm{S}(\mathrm{N}(R),V))\subseteq \dot{\mathrm{D}}(\mathrm{S}(\mathrm{N}(R'),V))$). Consequently, $R$ and $R'$ are equivalent ($R \equiv R'$) if and only if $\dot{\mathrm{D}}(\mathrm{S}(\mathrm{N}(R),V))\equiv\dot{\mathrm{D}}(\mathrm{S}(\mathrm{N}(R'),V))$.
\end{corollary}

\begin{proof}
    Let $R$ and $R'$ be rules such that $R\sqsubseteq R'$. This means that for any event $e$, if $\mathsf{match}(R,e)$ is true, then $\mathsf{match}(R',e)$ is true. Let us assume there exists a simple rule $\tau$ in $D=\mathrm{\dot D}(\mathrm{S}(\mathrm{N}(R),V))$ and not in $D'=\mathrm{\dot D}(\mathrm{S}(\mathrm{N}(R'),V))$ such that $\mathsf{match}(\tau,e)$ is true. Theorem \ref{theorem:sameRuleComparison} states that, for an event $e$, $\mathsf{match}(R,e)$ and $\mathsf{match}(R',e)$ are true if there exists a simple rule $\tau'$ in both $D$ and $D'$. However, if $\mathsf{match}(R,e)$, $\mathsf{match}(\tau,e)$ and $\mathsf{match}(\tau',e)$ are all true, then $\tau=\tau'$ holds by Theorem \ref{theorem:overlap}, which is a contradiction because $\tau$ does not appear in $D'$. Therefore, if for all events $e$, whenever $\mathsf{match}(R,e)$ is true, $\mathsf{match}(R',e)$ is true, there cannot exist a simple rule $\tau$ in $D$ and not in $D'$. Conversely, this means that for all simple rules $\tau$ such that $\mathsf{match}(\tau,e)$, either $\tau \notin D$ or $\tau \in D'$.
    For the other direction of the proof, assume $D\subseteq D'$ and there exists an event $e$ such that $\mathsf{match}(R,e)$ is true and $\mathsf{match}(R',e)$ is false. This means that for all $\tau \in D$, $\tau \in D'$. Furthermore, according to Corollary \ref{corollary:summary}, for any given rule $R$, event $e$ and values $V$, $\tau \in D$ and $\tau \in D'$ iff $\mathsf{match}(R,e)$ and $\mathsf{match}(R',e)$ are true. However, we assumed $\mathsf{match}(R',e)$ is false, so this is a contradiction.
\end{proof}

Finally, we claim that the result for Theorem \ref{theorem:sameRuleComparison} extends to sets of rules. We first formally define the set of simple \eventrules of the normalisation of a set of rules.

\begin{definition}
    For a set of rules $\mathbf{R}=\{R_1, R_2, \dots, R_n \}$ and a set of right operands $V$ 
    %$V=\bigcup_{i=1}^{n} \mathrm{V}(R_i)$, 
    we define $\mathrm{NS}(\mathbf{R},V)$ as the set of simple \eventrules $\bigcup_{R_i \in \mathbf{R}} \mathrm{\dot D}(\mathrm{S}(\mathrm{N}(R_i),V))$.
    \label{definition:set_of_rules}
\end{definition}

By using Definition \ref{definition:set_of_rules}, we can expand the result from Theorem \ref{theorem:sameRuleComparison} to apply to the normalisation of a set of rules.

\begin{corollary}
    Given a set of right operands $V$ and an event $e$, there exists a rule $R$ in set of rules $\mathbf{R}$ such that $\mathsf{match}(R,e)$ is true if and only if there exists a \eventrule $\tau$ in $\mathrm{NS}(\mathbf{R},V)$ 
    such that $\mathsf{match}(\tau,e)$ is true.
    \label{corollary:overlap_for_all_V}
\end{corollary}
\begin{proof} 
From Corollary \ref{corollary:summary} it follows that there exists a $\tau'$ in $\mathrm{D}(\mathrm{S}(\mathrm{N}(R),V))$ such that $\mathsf{match}(\tau',e)$, by Definition \ref{definition:set_of_rules} this $\tau'$ is included in $\mathrm{NS}(\mathbf{R},V)$. Notice that the $\dot{\mathrm{D}}$ function is equivalent to the $\mathrm{D}$ as it only involves the removal of redundant constraints that do not match any event.
\end{proof}
\begin{corollary}
    For any event $e$, there exists a rule $R$ in set of rules $\mathbf{R}$ such that $\mathsf{match}(R,e)$ is true if and only if there exists a \eventrule $\tau$ in $\mathrm{NS}(\mathbf{R},V)$, with $V \supset \bigcup_{R_i\in \mathbf{R}} \mathrm{V}(R_i)$, 
    such that $\mathsf{match}(\tau,e)$ is true and for no other \eventrule $\tau' \in \mathrm{NS}(\mathbf{R},V)$ is $\mathsf{match}(\tau',e)$ true.
    \label{corollary:overlap}
\end{corollary}
\begin{proof}
By Definition \ref{definition:set_of_rules}, each rule $R'$ in $\mathbf{R}$ is normalised into set $\dot{\mathrm{D}}(\mathrm{S}(\mathrm{N}(R'),V))$, of which, by Theorem \ref{theorem:overlap},  no more than one rule can match $e$. From Corollary \ref{corollary:summary} we know that there exists a $\tau$ in $\dot{\mathrm{D}}(\mathrm{S}(\mathrm{N}(R'),V))$ that matches $e$. For this corollary to be false, there should exist another rule $R''$, whose normalised set of rules $\dot{\mathrm{D}}(\mathrm{S}(\mathrm{N}(R''),V))$ contains a rule $\tau''$, different from $\tau$ that also matches $e$. However, this case would contradict Theorem \ref{theorem:sameRuleComparison}, which states that if two rules match the same event, then the rules in their decomposition that matches such event must be the same. In the other direction, $\mathsf{match}(\tau,e)$ implies that there exists a rule in $R$ that matches $e$, as already shown in Corollary \ref{corollary:overlap_for_all_V}.
\end{proof}

%\begin{proof}[sketch]
%Corollary \ref{corollary:overlap} follows by induction. Our base case is for $n=2$ which has been proven in Theorem \ref{theorem:sameRuleComparison}. Then for our inductive step, we have to consider for $\mathbf{R}=\{R_1,\dots,R_n\}$, an additional rule $R_{n+1}$ and the set of values $V'=V\cup \mathrm{V}(R_{n+1})$ and decide whether all simple \eventrules in $\mathrm{NS}(\mathbf{R}\cup R_{n+1},V')$ are disjoint. By induction, we assume this is true for $\mathrm{NS}(\mathbf{R},V)$, and by Theorem \ref{theorem:overlap}, this is true for $\mathrm{\dot D}(\mathrm{S}(\mathrm{N}(R_{n+1}),\mathrm{V}(R_{n+1})))$. Therefore, we only need to check if any simple \eventrule $\tau$ in $\mathrm{NS}(\mathbf{R},V')$ overlaps with a simple \eventrule $\tau'$ in $\mathrm{\dot D}(\mathrm{S}(\mathrm{N}(R_{n+1}),\mathrm{V}(R_{n+1})))$ such that $\tau \neq \tau'$. This can be argued in a similar manner to Theorem \ref{theorem:sameRuleComparison}.  
%\end{proof}

\subsection{Permissions and Prohibitions}

In the chosen semantics~\cite{salas2025evaluation} the compliance of a state of the world to a policy is expressed as a logical expression of matches between rules against events in the state of the world. A state of the world $\omega$ is valid w.r.t. an ODRL  policy $\langle\mathbf{P},\mathbf{F},\mathbf{O}\rangle$, where $\mathbf{P}$, $\mathbf{F}$ and $\mathbf{O}$ are sets of rules interpreted, respectively, as permissions, prohibitions and obligations, if:

\begin{itemize}
    \item for all events in the state of the world, there exists a permission that matches it, i.e. $\forall e \in \omega, \exists R \in \mathbf{P}, \ \mathsf{match}(R,e)$
    \item for all events in the state of the world, there is no prohibition that matches it, i.e. $\forall e \in \omega, \nexists R \in \mathbf{F}, \ \mathsf{match}(R,e)$
    \item for all obligations, there exists an event in the state of the world that is matched by it, i.e. $\forall R \in \mathbf{O}, \exists e \in \omega, \ \mathsf{match}(R,e)$
\end{itemize}

Our normalisation approach does not alter the semantics of permissions and prohibitions in a policy because, as shown above, an event matches a rule if and only if it matches one of the rules in its normalised form. So any event that is permitted or prohibited before the normalisation will still be permitted or prohibited afterwards. However, our normalisation cannot be used to normalise obligations in the same manner as permissions and prohibitions. This is because validity of obligations requires that for each obligation, there exists an event that is matched by it. Therefore, splitting an obligation into multiple ones effectively generates additional obligations, and thus a stricter policy. For example, if we were to normalise the obligation to pay an amount of at least 10 units $\{(Action=Pay) , (Amount>10)\}$, in the presence of a right operand with value ``20''. This could be normalised into the three obligations $\{(Action=Pay) , (Amount>10), (Amount<20)\}$, $\{(Action=Pay) , (Amount=20)\}$, $\{(Action=Pay) , (Amount>20)\}$, which can only be satisfied by three payments: one of 10 to 20 euros, exactly 20 euros, and more than 20 euros. 

\begin{theorem}
    Let $\langle\mathbf{P}, \mathbf{F}, \mathbf{O}\rangle$ be an ODRL policy and $V$ a set of right operands. A state of the world $\omega$ is valid w.r.t. a policy $\langle\mathbf{P}, \mathbf{F}, \mathbf{O}\rangle$ if and only if $\omega$ is valid w.r.t. an ODRL policy $\langle\mathrm{NS}(\mathbf{P},V), \mathrm{NS}(\mathbf{F},V), \mathbf{O})\rangle$.
    \label{theorem:equivalence_of_normalised_policies}
\end{theorem}

\begin{proof}

Each of the three conditions of validity evaluate a state of the world against $\mathbf{P}$, $\mathbf{F}$ and $\mathbf{O}$ in isolation from the other types of rules. Therefore the condition over obligations $\mathbf{O}$ is evaluated equivalently for both policies.  Thus, for a state of the world to be valid in one of the two policies, but not in the other, there must be a difference in how either the permissions or the prohibitions are evaluated. This can only occur if there exist an event $e$ that either A) matches only one between the sets of permissions $\mathbf{P}$ and $\mathrm{NS}(\mathbf{P},V)$ or B) matches only one between the sets of prohibitions $\mathbf{F}$ and $\mathrm{NS}(\mathbf{F},V)$. However, both cases A) and B) contradict Corollary \ref{corollary:overlap_for_all_V}.
\end{proof}

Under our interpretation, prohibitions are defined to ``carve out'' exception to permissions. To illustrate, imagine a permission for ``Alice to analyse health data'' but we want to exclude health data from before 2010. In such a case, we would define a prohibition ``Alice to analyse health data'' with an additional constraint that the date of collection of the data is less than 01-01-2010. Our normalisation approach would split the permission for ``Alice to analyse health data'' into three separate, but (equivalent when taken together) permissions, namely a permission for Alice to analyse health data if the date of collection is less than 01-01-2010, another one if the date of collection is exactly 01-01-2010, and lastly a third permission if the date of collection is larger than 01-01-2010. Notice how now, the rule of the prohibition to analyse health data from before 2010 is the same rule as the first permission in the normalisation. In this case, we can simply remove this matching pair of permission and prohibition, and the final, equivalent policy will contain only the two permissions for Alice to analyse health data if the date of collection is larger than 01-01-2010, and the one if the date of collection is exactly 01-01-2010. It is easy to see that the resulting policy is equivalent to the original one, and it does not contain prohibitions.

\begin{theorem}
    Given an ODRL Policy $\langle\mathbf{P},\mathbf{F},\mathbf{O}\rangle$ with $\mathbf{O}=\emptyset$, a state of the world $\omega$ is valid w.r.t. $\langle\mathbf{P},\mathbf{F},\mathbf{O}\rangle$ if and only if for all events $e \in \omega$, there exists a simple \eventrule $\tau \in \mathrm{NS}(\mathbf{P},V) \setminus \mathrm{NS}(\mathbf{F},V)$, with $V \supset \bigcup_{R_i\in \mathbf{P},\mathbf{F}} \mathrm{V}(R_i)$, such that $\mathsf{match}(\tau,e)$.
    \label{theorem:permissions_minus_prohibitions}
\end{theorem}

\begin{proof}
    $\Rightarrow$ If $\omega$ is valid w.r.t. $\langle\mathbf{P},\mathbf{F},\mathbf{O}\rangle$, then for all events $e \in \omega$ there exists a permission $P \in \mathbf{P}$ that matches $e$ and no prohibition $F \in \mathbf{F}$ that matches $e$, or equivalently, for all prohibitions in $\mathbf{F}$, $\mathsf{match}(F,e)$ is false. From Corollary \ref{corollary:summary} it follows that if there exists a permission $P$ such that $\mathsf{match}(P,e)$ is true, then there must exist a simple \eventrule $\tau \in \mathrm{\dot D}(\mathrm{S}(\mathrm{N}(P),V))$ such that $\mathsf{match}(\tau,e)$ is true. Moreover, since $\mathrm{NS}(\mathrm{\mathbf{P}},V)$ is a union that includes $\mathrm{\dot D}(\mathrm{S}(\mathrm{N}(P),V))$,  if such a simple \eventrule exists, then it must also exist in $\mathrm{NS}(\mathbf{P},V)$.  Conversely, if a prohibition $F$ doesn't match $e$, then there doesn't exist a simple \eventrule in $\mathrm{\dot D}(\mathrm{S}(\mathrm{N}(F),V))$ such that $\mathsf{match}(\tau,e)$ is true. Furthermore, since this is true for all prohibitions $F \in \mathbf{F}$, it must also be true for $\mathrm{NS}(\mathbf{F},V)$. We can conclude, that $\tau$ can be found in $\mathrm{NS}(\mathbf{P},V)$ and not in $\mathrm{NS}(\mathbf{F},V)$, and therefore exists in $\mathrm{NS}(\mathbf{P},V) \setminus \mathrm{NS}(\mathbf{F},V)$. \\
    $\Leftarrow$ If $\tau \in \mathrm{NS}(\mathbf{P},V) \setminus \mathrm{NS}(\mathbf{F},V)$ then this means that $\tau \in \mathrm{NS}(\mathbf{P},V)$ and $\tau \notin \mathrm{NS}(\mathbf{F},V)$. 
    Then, if $\tau \in \mathrm{NS}(\mathbf{P},V)$ it means there exists a rule $P \in \mathbf{P}$ such that $\tau \in \mathrm{\dot D}(\mathrm{S}(\mathrm{N}(P),V))$ for some $V$. Simultaneously, if $\tau \notin \mathrm{NS}(\mathbf{F},V)$, we know that for all $F \in \mathbf{F}$, $\tau \notin \mathrm{S}(\mathrm{N}(F),V)$. Because of Corollary \ref{corollary:overlap}, if $\mathsf{match}(\tau,e)$, then there exists a rule $P\in \mathbf{P}$ such that $\mathsf{match}(P,e)$ is true.
    Moreover, there can be no other \eventrule $\tau'$ in $\mathrm{NS}(\mathbf{P},V)$ or $\mathrm{NS}(\mathbf{F},V)$ such that $\mathsf{match}(\tau',e)$ is true. Given the above, the state of the world $\omega$ must be valid w.r.t. $\langle\mathbf{P},\mathbf{F},\mathbf{O}\rangle$.
\end{proof}
% TODO Add corollary where O is not empty
\begin{corollary}
 Given an ODRL Policy $\langle\mathbf{P},\mathbf{F},\mathbf{O}\rangle$, a state of the world $\omega$ is valid w.r.t. $\langle\mathbf{P},\mathbf{F},\mathbf{O}\rangle$ if and only if it is valid w.r.t. $\langle\mathrm{NS}(\mathbf{P},V) \setminus \mathrm{NS}(\mathbf{F},V),\emptyset,O\rangle$.
\label{corollary:permissions_minus_prohibitions}  
\end{corollary}

Corollary \ref{corollary:permissions_minus_prohibitions} expands the result of Theorem \ref{theorem:permissions_minus_prohibitions} to the case of non-empty sets of obligations for the same reason described in the proof of Theorem \ref{theorem:equivalence_of_normalised_policies}.
This is our main result that allows us to reformulate a policy with permissions and prohibitions to a policy of only permissions. Since our normalisation can be applied equally to permissions and prohibitions, our approach can be also used in the case of permission-by-default semantics, with the obvious change of removing permissions, rather than prohibitions, in the last step. 

\section{Complexity}

In this section, we analyse the complexity of the policy comparison decision problems through our normalisation method. Since our normalised rules can decide containment, equivalence or overlap by iterating over all the simple \eventrules that comprise the normalised rules, the size of the sets $\mathrm{D}(R,V)$ are the main determinants of the complexity. 
Let us assume a single rule $R$ has $n$ constraints. The most relevant metrics for our analysis are the size of the rules (i.e., the number of constraints) $n$, the number of distinct left operands (or attributes) that are mentioned in rules $L\leq n$, and the size of $V$, the distinct values appearing for each left operand.

For a single \eventrule of size $n$, the worst case scenario for the regularisation is when each constraint is a logical constraint in CNF form. Assume this logical constraint $c$ contains a total of $n$ constraints, where the simple constraints are distributed into $j$ conjunctions, each containing $k$ simple constraints. In other words, $c$ is a conjunction $(c_1\wedge c_2 \wedge \ldots \wedge c_j)$, where each $c_i$ is of the form $(x_1\vee x_2 \vee \ldots \vee x_k)$. Furthermore, if we assume each of the simple constraints $x_i$ is a constraint that after reformulation would produce an additional disjunction, $c_i$ would double in size to $2k$. The DNF expansion of $c$ would result in a disjunction of $(2k)^j$ conjunctions with $j$ simple constraints each, which adds up to $j \times (2k)^j$ simple constraints. This is maximised when $k$ is $e$ and $j$ is $\frac{n}{e}$, leading to a total of $e^{\frac{n}e{}}$ \eventrules, each containing a total of $\frac{n}{e}$ simple constraints. In practice, we can only have a natural number of constraints, so this is simplified to $k=\lfloor e \rfloor=2$ and $j=\frac{n}{2}$, and thus we have $(2\times2)^\frac{n}{2}=2^n$ permissions, each of which has at least $\frac{n}{2}$ constraints in addition to the core ODRL elements.
% In practice, this is worst-case scenario is not likely to appear in the real world; it would require that each constraint be a logical constraint in CNF form.

The worst case for the splitting of intervals is when rule $R$ has no constraints other than those for the main ODRL elements (i.e., it allows any value for all other left operands) and $R'$ has $L=n$ distinct left operands. In such a case, $R$ will be split into $\prod_{i=0}^L 2|V(R',\lambda_i)|+1$, where $|V(R',\lambda_i)|$ is the number of distinct values for $\lambda_i$ in $R'$. Assuming the total number of values is $|V|$ and the number of distinct left operands is $L$, this is maximised when each $|V(R',\lambda_i)|$ is $\frac{|V|}{L}$ (for simplicity, assume $|V|=kL$ for some integer). In such a case, the total number of split intervals is $(2\frac{|V|}{L}+1)^L$, which is polynomial on $|V|$ and exponential on $L$. 
Note that in this scenario, the DNF expansion of $R$ would yield simply $R$ because there are no constraints to be reformulated. Therefore, this worst case scenario is not directly related to the worst case scenario for the DNF expansion.
On the other hand, the best case scenario is when all constraints are equalities and both rules mention the same left operands, because in such a case we don't have to split any intervals.

For a set of rules $\mathbf{R}$, its size is determined by the sizes of the normalisation of each of its rules. If we assume that each $R_i \in \mathbf{R}$ is disjoint with $R_j \in \mathbf{R}$ such that $R_i \neq R_j$, the total size of $\mathrm{NS}(R,V)$ is $\sum_{R \in \mathbf{R}}|\mathrm{D}(R,V)|$. In other words, it grows linearly with the number of rules in $\mathbf{R}$ and exponential with the size of each rule, leading to a size of $O(|\mathbf{R}|\cdot 2^n)$.

\section{Related Work}

To our knowledge, ours is the first that addresses the problem of normalising ODRL policies, and provides a comparison of ODRL policies via the normalisation of its components. We now review the most relevant pieces of work in the literature that address similar semantics problems. Garc{\'\i}a et al.~\cite{garcia2005formalising} employed semantic web technologies such as OWL~\cite{mcguinness2004owl}, to translate ODRL (in XML format) to OWL ontologies and RDF graphs, which allowed them to leverage the IPROnto ontology for digital rights management~\cite{delgado2003ipronto}. However, their approach was tailored for ODRL 1.1, which is deprecated, and lacks the richer tapestry of rule types such as prohibitions, and obligation. 
Pucella and Weissman~\cite{pucella2006formal} proposed a translation of ODRL 1.1 Policies to logical formulas. The normalisation steps discussed in this study could also be applied to their constraints, which would also allow logical normal forms.  
Steyskal and Polleres~\cite{steyskal2015towards} earliest proposed semantics for ODRL 2.0+ based on rule-matching and logic programs; notably their semantics took into consideration the implicit dependencies between actions. 
Slabbinck et al.~\cite{slabbinck2025interoperable} defined a vocabulary to describe interoperable ODRL Evaluation reports and also presented their ODRL Evaluator, a piece of software that can evaluate policies for access control scenarios. 
Kebede et al.\cite{kebede2018critical} discuss the representational power of ODRL, reinforced by presenting use-cases and examples, and highlights some of the limitations of ODRL. For example, the ambiguous semantics of duties (as per the currently published recommendations) or the granularity of parties (e.g., certain members of an organisation). 
De Vos et al.~\cite{de2019odrl} implement policy and compliance checking by translating policies into answer-set programs to check for compliance. The semantics of answer-set programming are ideal for applications that adopt a closed-world assumption. Their proposed model is ideal for representing real-world legal frameworks such as the GDPR. 
Kieffer et al.~\cite{kieffer2025composing} propose an approach that computes the least restrictive license that complies to a set of policies. If such a license does not exist, one can assume that there are no events that are valid w.r.t. all policies simultaneously. Therefore, said approach could be used to verify policy overlap. However, the problems of containment and equivalence were not a focus of their work, and cannot be trivially derived.

\section{Conclusion}

We have presented a procedure to compute a normal form for ODRL rules, that preserves their original evaluation semantics, while yielding useful properties, following a recently published formalisation of ODRL evaluation. 
Two key steps of this procedure involve decomposing rules with complex logic constraints into rules with simple conjunctions of basic constraints, and splitting constraints over numeric intervals into their basic components, with respect to a given set of known constants. This second step  
guarantees that each simple \eventrule in the normal form is disjoint, which allows to decide containment, equivalence or overlap between rules by checking equality of pairs of simple \eventrules instead of combinations of simple \eventrules. 
The useful properties of the normalisation procedure can also be applied to sets of ODRL rules, and therefore ODRL policies with permissions and prohibitions. 
Notably, only one set of rules between permissions and prohibitions needs to be retained, and the other can be removed completely, leaving only a set of disjoint permissions (or prohibitions) that can then be checked for matches to events. This, combined with the simplification of logic combinations of constraints, provides a mechanism to translate complex policies into equivalent ones that do not require logic constraints, and only require one type of rule between permissions and prohibitions, increasing the interoperability between tools with levels of support for ODRL features. Our normalisation targets core elements of ODRL, but additional features are left for future work, such as additional rule types, constraints over sets, and the interaction with inference.

\bibliography{main}
\bibliographystyle{splncs04}

\end{document}